\documentclass{article}
\usepackage{ijcai26}

\usepackage{times}
\usepackage{soul}
\usepackage{bm}
\usepackage{url}
\usepackage[hidelinks]{hyperref}
\usepackage[utf8]{inputenc}
\usepackage[small]{caption}
\usepackage{graphicx}
\usepackage{amsmath}
\usepackage{amsthm}
\usepackage{booktabs}
\usepackage{algorithm}
\usepackage{algorithmic}
\usepackage[switch]{lineno}
\usepackage{amssymb}
\usepackage{multirow}
\usepackage{algorithm}
\usepackage{algorithmic}

\usepackage{pifont} 
\usepackage{subcaption}
\usepackage{verbatim}
\usepackage{graphicx}
\usepackage{multirow}
\usepackage{xspace}
\usepackage{amsfonts,amssymb} 
\usepackage{microtype}
\usepackage{amsmath}
\usepackage{amsthm}
\usepackage{booktabs}
\usepackage{times}
\usepackage{latexsym}
\usepackage{color}
\usepackage[dvipsnames, table]{xcolor}
\usepackage{soul}
\usepackage{verbatim}
\usepackage{multirow}
\usepackage{xspace}
\usepackage{amsfonts,amssymb} 
\usepackage{microtype}
\usepackage{times}
\usepackage{latexsym}
\usepackage{color}
\usepackage[dvipsnames, table]{xcolor}
\usepackage{soul}
\usepackage{colortbl}
\usepackage{booktabs,makecell, multirow, tabularx}
\usepackage{adjustbox}
\usepackage{xcolor}
\usepackage{listings}
\usepackage{amsmath}

\usepackage{booktabs}
\usepackage{amssymb}
\usepackage{bbding}
\usepackage{pifont}
\usepackage{wasysym}
\usepackage{utfsym}
\usepackage{fontawesome}

\usepackage{mdframed}
\usepackage{mathtools}

\definecolor{color_gray}{RGB}{229,229,229}
\definecolor{color_blue}{RGB}{252,182,165}
\definecolor{color_pink}{RGB}{255,217,178}
\definecolor{color_yellow}{RGB}{255,255,204}
\definecolor{color_blue1}{RGB}{135, 206, 235}
\cellcolor{color_blue} \cellcolor{color_pink} \cellcolor{color_yellow} \cellcolor{blue!8}

\newtheorem{theorem}{Theorem}

\title{Breaking Semantic Hegemony: Decoupling Principal and Residual Subspaces for Generalized OOD Detection}

\author{
Ningkang Peng$^1$\and
Xiaoqian Peng$^2$\and
Yuhao Zhang$^1$\and
Qianfeng Yu$^1$\and
Feng Xing$^1$\and\\
Peirong Ma$^{1,*}$\and
Xichen Yang$^{1,*}$\and
Yi Chen$^{1,*}$\and
Tingyu Lu$^{3,*}$\and
Yanhui Gu$^{1}$\thanks{Corresponding authors.}
\\
\affiliations
$^1$School of Computer and Electronic Information, Nanjing Normal University, China\\
$^2$School of Artificial Intelligence and Information Technology, Nanjing University of Chinese Medicine, China\\
$^3$Advanced Institute for Materials Research, Tohoku University, Japan\\
\emails
\{prma, xichen\_yang, cs\_chenyi\}@njnu.edu.cn,
tingyu.lu.e7@tohoku.ac.jp,
gu@njnu.edu.cn
}

\begin{document}

\maketitle
\begin{abstract}
While feature-based post-hoc methods have made significant strides in Out-of-Distribution (OOD) detection, we uncover a counter-intuitive Simplicity Paradox in existing state-of-the-art (SOTA) models: these models exhibit keen sensitivity in distinguishing semantically subtle OOD samples but suffer from severe Geometric Blindness when confronting structurally distinct yet semantically simple samples or high-frequency sensor noise. We attribute this phenomenon to Semantic Hegemony within the deep feature space and reveal its mathematical essence through the lens of Neural Collapse. Theoretical analysis demonstrates that the spectral concentration bias, induced by the high variance of the principal subspace, numerically masks the structural distribution shift signals that should be significant in the residual subspace. To address this issue, we propose D-KNN, a training-free, plug-and-play geometric decoupling framework. This method utilizes orthogonal decomposition to explicitly separate semantic components from structural residuals and introduces a dual-space calibration mechanism to reactivate the model's sensitivity to weak residual signals. Extensive experiments demonstrate that D-KNN effectively breaks Semantic Hegemony, establishing new SOTA performance on both CIFAR and ImageNet benchmarks. Notably, in resolving the Simplicity Paradox, it reduces the FPR95 from 31.3\% to 2.3\%; when addressing sensor failures such as Gaussian noise, it boosts the detection performance (AUROC) from a baseline of 79.7\% to 94.9\%.
\end{abstract}

\begin{figure}
    \centering
    \includegraphics[width=1\linewidth]{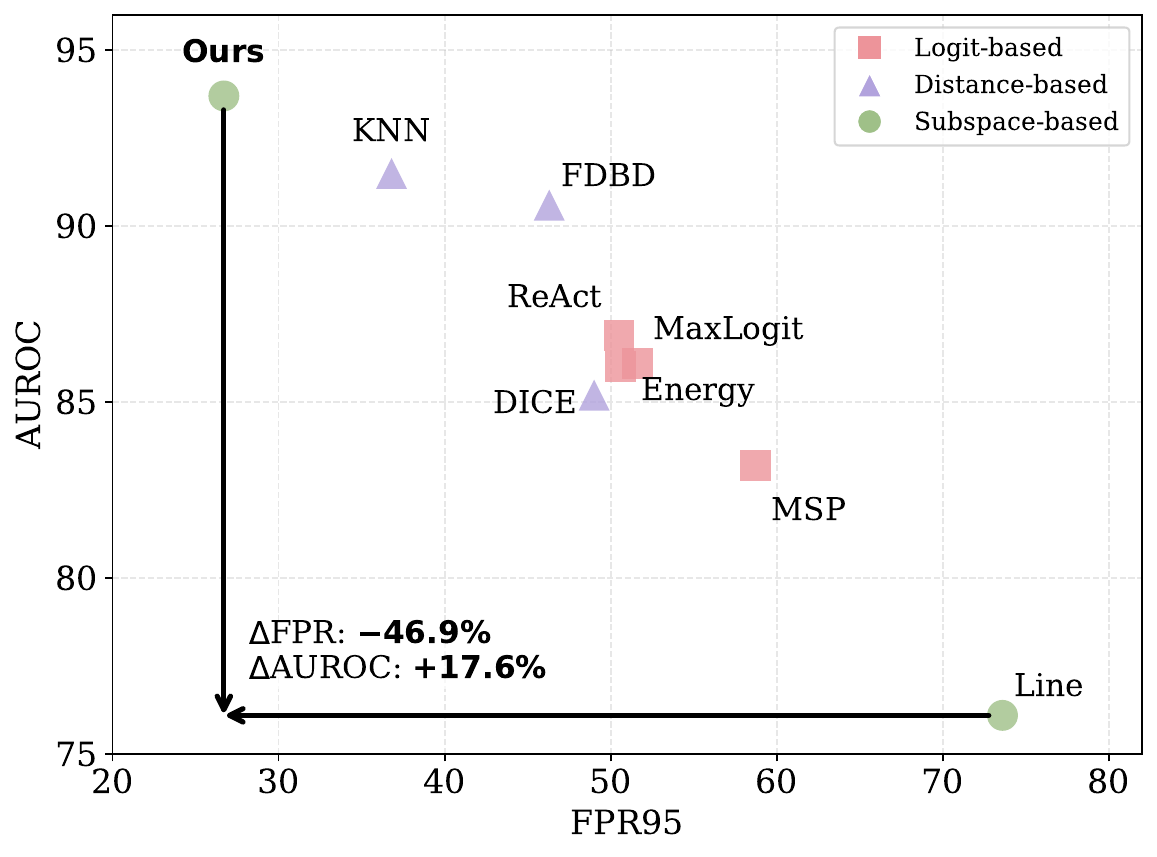}
    \caption{\textbf{Performance comparison of OOD detection methods on CIFAR-100.} We categorize existing methods into three paradigms: Logit-based (pink squares), Distance-based (purple triangles), and Subspace-based (green circles). }
    \label{fig:placeholder}
\end{figure}

\section{Introduction}

With the widespread deployment of Deep Neural Networks (DNNs) in safety-critical open-world applications, such as autonomous driving and medical diagnostics~\cite{feng2025survey,khan2025comprehensive,multiOOD}, Out-of-Distribution (OOD) detection has emerged as a paramount task~\cite{xu2025large,ijcai}. Feature-based post-hoc methods have become the mainstream benchmark, favored for their ability to leverage the powerful representation capabilities of pre-trained models without the need for expensive retraining~\cite{feature-distance-based2,wang2022vim}.

However, we reveal a pervasive yet under-explored counter-intuitive phenomenon: the \textit{Simplicity Paradox}. While existing SOTA methods excel at distinguishing semantically complex OOD samples, their detection performance deteriorates unexpectedly when encountering structurally distinct yet semantically simple samples (e.g., MNIST digits, texture patterns) or high-frequency sensor noise~\cite{winkens2020contrastive}. This defect is fatal in safety-critical domains, as it implies that models may remain oblivious to severe sensor failures, misclassifying them as known classes with high confidence.

\begin{center} \textit{\textbf{Why do simple samples trigger high-confidence Semantic Hallucinations?}} \end{center}

To explore the root cause of this paradox, we conduct a spectral analysis on deep features. We observe that, driven by the \textit{Neural Collapse} (NC) phenomenon~\cite{papyan2020prevalence}, feature energy at the terminal phase of training is heavily compressed into a low-dimensional \textit{Principal Subspace} spanned by class means, with eigenvalues exhibiting a sharp exponential decay. We define this phenomenon as \textbf{Semantic Hegemony}: principal components dominate the distance metrics, numerically masking the \textit{Residual Components} that carry crucial geometric structural information. Consequently, OOD samples that project close to In-Distribution (ID) data in the semantic space but possess large deviations in the residual space fall into the geometric blind spot of standard distance measures~\cite{sun2021react,chen2023importance}.

To address this issue, we propose a simple yet versatile framework: \textbf{Dual-Space KNN (D-KNN)}. Rather than passively accepting neural collapse, we proactively exploit this geometric property. Unlike traditional methods that calculate distances directly in the raw mixed space, D-KNN employs orthogonal decomposition to explicitly decouple features into a principal space carrying semantics and a residual space carrying domain structure~\cite{hendrycks2018deep,energy}. Through a dual-space calibration mechanism, we force the model to attend to anomalous information in the residual space, effectively pulling hidden outliers out of the blind spot during the inference phase.

The main contributions of this study are summarized as follows:
\begin{itemize}
    \item \textbf{Mechanism Discovery:} We reveal the cause of the \textit{Simplicity Paradox} from a geometric perspective, identifying \textit{Semantic Hegemony} in the feature space as the fundamental reason why existing methods fail on simple and noisy samples.
    \item \textbf{Methodological Innovation:} We propose D-KNN, a non-parametric dual-space decoupling framework. We theoretically prove that under the ideal Neural Collapse hypothesis, this method guarantees the asymptotic separability of ID and OOD samples.
    \item \textbf{Performance Verification:} Extensive experiments demonstrate that D-KNN effectively resolves the Simplicity Paradox on CIFAR and ImageNet benchmarks. Notably, it reduces the FPR95 on MNIST from 31.3\% to an astounding 2.3\%. Furthermore, D-KNN exhibits superior robustness in simulated sensor failure scenarios, establishing a new SOTA.
\end{itemize}

\section{Related Work}
\label{sec:related_work}
\subsection{Generalized Post-hoc OOD Detection}
Early works in OOD detection primarily \textbf{rely} on the model's output probabilities or logits. MSP~\cite{c:25} pioneers the use of the maximum Softmax probability as a baseline score. Subsequently, ODIN~\cite{confidence-score2} and Energy~\cite{energy} further amplify the distributional discrepancy between ID and OOD data by introducing temperature scaling and free energy functions, respectively. While these methods demonstrate strong performance on tasks with obvious semantic shifts, they inherently depend on the model's response to semantic content. Consequently, when confronting structurally distinct but semantically vacuous simple samples~\cite{nalisnick2018deep}, these methods often yield high-confidence Semantic Hallucinations due to the over-generalization of deep features, leading to detection failures.

\subsection{Distance-based and Subspace Methods}
To leverage richer geometric information within the feature space, distance-based methods have emerged. Mahalanobis~\cite{feature-distance-based} assumes that features follow a Gaussian distribution and utilizes class-conditional covariance to detect anomalies. KNN-OOD~\cite{feature-distance-based2} employs non-parametric nearest neighbor distances to directly characterize local density, thereby circumventing the limitations of distributional assumptions. In the realm of subspace analysis, ViM~\cite{wang2022vim} innovatively combines principal space logits with the residual space norm.
However, these methods typically overlook the issue of energy imbalance within the feature space. As demonstrated in our theoretical analysis, standard KNN computes distances directly in the raw space, causing faint structural cues to be masked by massive semantic variances (i.e., Semantic Hegemony). 

\subsection{Neural Collapse and Feature Geometry}
Neural Collapse~\cite{papyan2020prevalence} reveals a pervasive geometric phenomenon in deep networks during the terminal phase of training: intra-class feature variance vanishes, and feature centers collapse into a Simplex Equiangular Tight Frame (ETF)~\cite{strohmer2003grassmannian,han2021neural}. Recently, researchers have begun to explore the nexus between NC and OOD detection. For instance, Neco~\cite{yang2023neural} utilizes NC properties to constrain feature embeddings. However, existing research primarily focuses on leveraging NC to enhance classification performance; no prior work has linked the Spectral Concentration Bias induced by NC to the failure of OOD detection (the Simplicity Paradox). This study is the first to identify that while NC benefits classification (via principal space separability), the resulting drastic energy compression is the very root of the model's blindness to simple geometric structures, motivating our proposed decoupling-based remedy.

\section{Methodology}

\subsection{Motivation and Preliminaries }
Our approach is motivated by the geometric properties of deep features in the Terminal Phase of Training (TPT)~\cite{papyan2020prevalence} and the inherent limitations of standard distance-based metrics.

\noindent\textbf{Neural Collapse (NC).}
Recent theoretical works reveal that models trained with Cross-Entropy loss converge to a rigid geometric structure known as Neural Collapse. Specifically, the within-class covariance $\boldsymbol{\Sigma}_W$ vanishes, and the class means $\{\boldsymbol{\mu}_c\}$ form a Simplex Equiangular Tight Frame (ETF):
\begin{equation}
    \boldsymbol{\Sigma}_W = \sum_{c=1}^C \sum_{\mathbf{z} \in \text{class } c} (\mathbf{z} - \boldsymbol{\mu}_c)(\mathbf{z} - \boldsymbol{\mu}_c)^\top \to \mathbf{0}.
\end{equation}
This implies that feature energy is aggressively compressed into a low-rank subspace spanned by the class means, while the orthogonal (residual) directions are suppressed.

\noindent\textbf{Standard KNN for OOD Detection.}
Traditional feature-based methods rely on the Euclidean distance in the feature space. For a test sample $\mathbf{z}$, the anomaly score is defined as the distance to its $k$-th nearest neighbor $\mathbf{n}_k$:
\begin{equation} \label{eq:std_knn}
    S_{\text{KNN}}(\mathbf{z}) = \|\mathbf{z} - \mathbf{n}_k\|_2.
\end{equation}

\noindent\textbf{Semantic Hegemony and Geometric Blind Spot.}
We define the Semantic Hegemony Ratio $\rho$ as the ratio of energy in the principal subspace versus the residual subspace:
\begin{equation}
    \rho = \frac{\sum_{i=1}^d \lambda_i}{\sum_{j=d+1}^D \lambda_j}.
\end{equation}
As $\boldsymbol{\Sigma}_W \to \mathbf{0}$, $\rho \to \infty$. In this scenario, the distance metric in Equation~\eqref{eq:std_knn} becomes dominated by the principal components. Consider an OOD sample $\mathbf{z}_{out} = \mathbf{z}_{in} + \boldsymbol{\delta}_{res}$, where $\boldsymbol{\delta}_{res}$ is a structural anomaly in the residual subspace. Since $\|\boldsymbol{\delta}_{res}\| \ll \|\mathbf{z}_{in}\|$, standard KNN will perceive $\mathbf{z}_{out} \approx \mathbf{z}_{in}$, leading to a Geometric Blind Spot.

\subsection{Dual-Space Geometric Decoupling}
To eliminate the blind spot, we construct two orthogonal views of the feature space: one capturing the dominant semantic structure and the other isolating the suppressed geometric variations.

\noindent\textbf{Hyperspherical Projection.}
Firstly, we project all features onto the unit hypersphere $\mathbb{S}^{D-1}$ via $\phi(\mathbf{x}) = \mathbf{x}/\|\mathbf{x}\|_2$, performing L2 normalization. Let $\mathcal{Z} = \{\mathbf{z}_i\}_{i=1}^N$ denote the centered and normalized training features.

\noindent\textbf{Orthogonal Decomposition.}
We identify the principal subspace by solving for the low-rank basis $\mathbf{V} \in \mathbb{R}^{D \times d}$ that minimizes reconstruction error on $\mathcal{Z}$ via PCA. This yields two orthogonal projection operators:
\begin{equation}
    \mathbf{P}_{prin} = \mathbf{V}\mathbf{V}^\top, \quad \mathbf{P}_{res} = \mathbf{I} - \mathbf{P}_{prin},   
\end{equation}
where the Residual Space is the orthogonal complement of the Principal Space.

\noindent\textbf{Manifold Retraction.}
To recover structural patterns, we introduce a Manifold Retraction strategy. For any sample $\mathbf{z}$ and the global feature mean $\boldsymbol{\mu}$, we construct two independent views:
\begin{equation}
\begin{split}
    \mathbf{z}_{prin} &= \phi \left( (\mathbf{z} - \boldsymbol{\mu})\mathbf{P}_{prin} + \boldsymbol{\mu} \right), \\
    \mathbf{z}_{res} &= \phi \left( (\mathbf{z} - \boldsymbol{\mu})\mathbf{P}_{res} + \boldsymbol{\mu} \right).
\end{split}
\end{equation}
The term $(\mathbf{z} - \boldsymbol{\mu})\mathbf{P}_{res}$ isolates the residual vector within the tangent space of the semantic manifold. By reintegrating the global anchor $\boldsymbol{\mu}$ and performing hyperspherical re-projection, this operation effectuates a manifold retraction of the residual component. This transformation maps abstract energy magnitudes into directional structural descriptors, thereby enhancing the separability of anomalies within the suppressed residual subspace.

\subsection{Dual-Space Calibration}
We perform density estimation in the decoupled spaces. For a view $v \in \{p, r\}$, the score is the Euclidean distance to the $k$-th nearest neighbor:
\begin{equation}
    s_v(\mathbf{z}) = \|\mathbf{z}_v - \mathbf{n}_k^{(v)}\|_2.
\end{equation}

\noindent\textbf{Z-score Calibration.}
To address the scale mismatch caused by Semantic Hegemony, we employ Dual-Space Calibration:
\begin{equation} \label{eq:calibration}
    \tilde{s}_v(\mathbf{z}) = \frac{s_v(\mathbf{z}) - \mu_v}{\sigma_v},
\end{equation}
where $\mu_v$ and $\sigma_v$ are the mean and standard deviation of ID scores. Under the NC regime ($\rho \to \infty$), the vanishing denominator $\sigma_r$ acts as a \textbf{differential amplifier}, significantly amplifying even minute deviations in the residual structure.

\noindent\textbf{Final Decision.}
The final anomaly score is a weighted fusion of the calibrated scores:
\begin{equation}
    S(\mathbf{z}) = \alpha \cdot \tilde{s}_p(\mathbf{z}) + (1-\alpha) \cdot \tilde{s}_r(\mathbf{z}).
\end{equation}
A higher $S(\mathbf{z})$ indicates a higher probability of being OOD.

\subsection{Algorithm Summary}
The procedure for Dual-Space KNN (D-KNN) is summarized below.

\begin{algorithm}[H]
\caption{Dual-Space KNN (D-KNN)}
\begin{algorithmic}[1]
\REQUIRE Training set $\mathcal{X}$, Test sample $\mathbf{x}$, Neighbors $K$, Weight $\alpha$.
\STATE \textbf{Definitions:} $\phi(\mathbf{v}) \coloneqq \mathbf{v} / \|\mathbf{v}\|_2$; $D_k(\mathbf{z}, \mathcal{M})$ is the $k$-th NN Euclidean distance.
\STATE \textbf{Offline Phase:}
\STATE $\mathcal{Z} \leftarrow \{\phi(\mathbf{x}_i)\}_{i=1}^n$, $\boldsymbol{\mu} \leftarrow \text{mean}(\mathcal{Z})$.
\STATE Compute $\mathbf{P}$ and $\mathbf{R}$ via PCA on $\mathcal{Z}$.
\STATE Build dual galleries $\mathcal{M}_v = \{ \phi( \mathbf{P}_v(\mathbf{z}_i - \boldsymbol{\mu}) + \boldsymbol{\mu} ) \}_{i=1}^n$ for $v \in \{p, r\}$.
\STATE Compute $(\mu_v, \sigma_v)$ from $\{D_k(\mathbf{m}, \mathcal{M}_v \setminus \{\mathbf{m}\})\}_{\mathbf{m} \in \mathcal{M}_v}$.
\STATE \textbf{Online Inference:}
\STATE Project test sample: $\mathbf{z}_v = \phi( \mathbf{P}_v(\phi(\mathbf{x}) - \boldsymbol{\mu}) + \boldsymbol{\mu} ), \quad v \in \{p, r\}$.
\STATE Calibrate scores: $\tilde{s}_v = (D_k(\mathbf{z}_v, \mathcal{M}_v) - \mu_v) / \sigma_v, \quad v \in \{p, r\}$.
\STATE \textbf{Final Decision:} $S = \alpha \tilde{s}_p + (1 - \alpha) \tilde{s}_r$.
\ENSURE Anomaly Score $S$.
\end{algorithmic}
\end{algorithm}

\subsection{Theoretical Analysis: Asymptotic Separability}
\begin{theorem}[Asymptotic Risk Vanishing]
Under the ideal NC regime ($\sigma_{\text{in}} \to 0$), assuming a fundamental gap exists such that $\mu_{\text{out}} - \mu_{\text{in}} \ge \Delta > 0$, the detection risk $\mathcal{R}$ for the calibrated score $\tilde{S}$ vanishes asymptotically.
\end{theorem}

\begin{proof}
Calibration maps ID scores to $\mathbb{E}[\tilde{S}_{\text{in}}] = 0$ and $\text{Var}(\tilde{S}_{\text{in}}) = 1$. For the OOD distribution, the expected calibrated score shifts to:
\begin{equation}
    \mathbb{E}[\tilde{S}_{\text{out}}] = \frac{\mu_{\text{out}} - \mu_{\text{in}}}{\sigma_{\text{in}}} \ge \frac{\Delta}{\sigma_{\text{in}}}.
\end{equation}
Let $\tau = \frac{\Delta}{2\sigma_{\text{in}}}$ be the threshold. By \textbf{Chebyshev's Inequality}:
\begin{equation}
    P(\tilde{s} \ge \tau \mid \mathcal{D}_{\text{in}}) \le \frac{\text{Var}(\tilde{S}_{\text{in}})}{\tau^2} = \frac{4\sigma_{\text{in}}^2}{\Delta^2}.
\end{equation}
As $\sigma_{\text{in}} \to 0$ (NC1 property), we have:
\begin{equation}
    \lim_{\sigma_{\text{in}} \to 0} \mathcal{R} \le \lim_{\sigma_{\text{in}} \to 0} \frac{4\sigma_{\text{in}}^2}{\Delta^2} = 0.
\end{equation}
This confirms that dual-space calibration acts as a \textbf{numerical singularity} for OOD samples.
\end{proof}

\begin{table*}[t]
\centering
\resizebox{\textwidth}{!}{%
\begin{tabular}{l l ccccccc c}
\toprule
\multirow{3}{*}{\textbf{ID}} & \multirow{3}{*}{\textbf{Method}} & \multicolumn{7}{c}{\textbf{OOD Datasets}} & \multirow{3}{*}{\textbf{Average}} \\
\cmidrule(lr){3-9}
& & EMNIST & MNIST & SVHN & Texture & LSUN-R & LSUN & iSUN & \\
& & FPR$\downarrow$ / AUC$\uparrow$ & FPR$\downarrow$ / AUC$\uparrow$ & FPR$\downarrow$ / AUC$\uparrow$ & FPR$\downarrow$ / AUC$\uparrow$ & FPR$\downarrow$ / AUC$\uparrow$ & FPR$\downarrow$ / AUC$\uparrow$ & FPR$\downarrow$ / AUC$\uparrow$ & FPR$\downarrow$ / AUC$\uparrow$ \\
\midrule
\multirow{11}{*}{\rotatebox{90}{CIFAR-10}} 
& MSP & 17.2 / 97.7 & 10.8 / 98.5 & 55.1 / 90.0 & 63.7 / 87.2 & 48.3 / 92.7 & 22.9 / 97.0 & 54.0 / 91.0 & 38.9 / 93.4 \\
& Energy & 1.0 / 99.6 & 0.1 / 99.8 & 46.1 / 88.2 & 59.0 / 84.2 & 23.6 / 95.3 & 2.9 / 99.3 & 30.6 / 93.8 & 23.3 / 94.3 \\
& MaxLogit & 1.3 / 99.5 & 0.2 / 99.8 & 45.2 / 88.3 & 57.9 / 84.4 & 24.2 / 95.2 & 3.3 / 99.2 & 31.0 / 93.7 & 23.3 / 94.3 \\
& Mahalanobis & 42.2 / 94.7 & 26.2 / 96.1 & 16.8 / 96.9 & 20.6 / 96.0 & 59.0 / 89.9 & 55.8 / 90.8 & 58.2 / 89.4 & 39.8 / 93.4 \\
& ViM & 25.4 / 95.2 & 18.5 / 96.9 & 15.6 / 96.6 & 20.1 / 96.5 & 16.5 / 97.2 & 14.2 / 97.3 & 18.2 / 96.2 & 18.4 / 96.5 \\
& ReAct & 1.2 / 99.6 & 0.2 / 99.8 & 46.5 / 88.4 & 57.4 / 86.2 & 24.4 / 95.0 & 3.2 / 99.2 & 31.6 / 93.5 & 23.5 / 94.5 \\
& DICE & 0.02 / 99.9 & 0.01 / 99.9 & 47.3 / 86.0 & 60.1 / 81.1 & 22.8 / 95.1 & 0.5 / 99.7 & 29.3 / 93.7 & 22.9 / 93.6 \\
& Line & 4.7 / 98.6 & 1.2 / 99.7 & 73.5 / 72.4 & 75.4 / 69.8 & 61.2 / 81.3 & 11.0 / 96.9 & 69.3 / 78.1 & 42.3 / 85.5 \\
& FDBD & 6.6 / 99.0 & 2.4 / 99.5 & 38.4 / 92.5 & 48.5 / 89.0 & 32.8 / 93.4 & 13.0 / 97.7 & 38.9 / 91.4 & 25.8 / 94.6 \\
& KNN & 30.9 / 95.3 & 24.2 / 96.3 & 31.3 / 94.9 & 40.6 / 92.3 & 27.5 / 95.0 & 26.8 / 95.3 & 31.0 / 93.9 & 30.3 / 94.7 \\
\rowcolor{gray!10} & \textbf{D-KNN (Ours)} & \textbf{2.5 / 99.4} & \textbf{2.1 / 99.5} & \textbf{17.2 / 96.7} & \textbf{28.8 / 94.5} & \textbf{26.2 / 94.9} & \textbf{15.7 / 97.2} & \textbf{30.0 / 93.9} & \textbf{17.5 / 96.6} \\
\midrule
\multirow{11}{*}{\rotatebox{90}{CIFAR-100}} 
& MSP & 25.2 / 95.2 & 2.8 / 99.4 & 84.8 / 71.4 & 86.1 / 71.4 & 80.5 / 78.7 & 49.0 / 88.6 & 82.3 / 77.7 & 58.7 / 83.2 \\
& Energy & 1.6 / 99.4 & 0.01 / 99.8 & 87.2 / 71.4 & 85.3 / 72.8 & 79.9 / 81.6 & 16.7 / 96.8 & 83.8 / 80.1 & 50.6 / 86.0 \\
& MaxLogit & 3.4 / 99.2 & 0.01 / 99.8 & 86.2 / 71.7 & 84.9 / 73.1 & 79.3 / 81.9 & 21.4 / 96.2 & 83.1 / 80.4 & 51.6 / 86.1 \\
& Mahalanobis & 100.0 / 48.2 & 100.0 / 43.4 & 44.3 / 89.9 & 42.3 / 90.0 & 38.2 / 91.7 & 99.7 / 47.6 & 40.5 / 91.2 & 72.1 / 71.7 \\
& ViM & 77.3 / 80.5 & 97.0 / 66.8 & 69.8 / 75.1 & 79.1 / 67.9 & 94.5 / 59.0 & 91.4 / 65.4 & 94.0 / 58.8 & 86.2 / 67.6 \\
& ReAct & 1.5 / 99.4 & 0.01 / 99.8 & 87.3 / 75.0 & 84.3 / 76.2 & 79.8 / 81.3 & 17.1 / 96.8 & 83.5 / 79.9 & 50.5 / 86.9 \\
& DICE & 0.02 / 99.7 & 0.01 / 99.8 & 90.4 / 68.6 & 83.0 / 72.4 & 81.6 / 78.8 & 4.2 / 98.7 & 83.6 / 78.2 & 49.0 / 85.2 \\
& Line & 2.6 / 99.2 & 0.01 / 99.8 & 96.8 / 62.5 & 81.3 / 69.0 & 98.2 / 53.3 & 38.1 / 92.4 & 98.9 / 56.3 & 73.6 / 76.1 \\
& FDBD & 12.9 / 98.0 & 0.4 / 99.8 & 72.0 / 83.6 & 72.0 / 83.3 & 58.0 / 89.8 & 45.8 / 91.2 & 63.2 / 88.2 & 46.3 / 90.6 \\

& KNN & 23.0 / 96.3 & 31.3 / 95.4 & 42.1 / 90.3 & 41.0 / 89.7 & 29.4 / 94.3 & 56.2 / 82.0 & 34.4 / 92.3 & 36.8 / 91.5 \\
\rowcolor{gray!10} & \textbf{D-KNN (Ours)} & \textbf{1.9 / 99.5} & \textbf{2.3 / 99.3} & \textbf{32.0 / 93.4} & \textbf{35.6 / 91.7} & \textbf{29.9 / 94.3} & \textbf{51.3 / 85.4} & \textbf{33.8 / 92.5} & \textbf{26.7 / 93.7} \\
\bottomrule
\end{tabular}
}
\caption{\textbf{Detailed OOD Detection Results on CIFAR-10 and CIFAR-100 (WRN-28-10).} \textbf{Metric:} FPR95 ($\downarrow$) / AUC ($\uparrow$). All results are percentages (\%).}
\label{tab:combined_results_updated}
\end{table*}

\begin{table*}[t]
\centering

\resizebox{\textwidth}{!}{%
\begin{tabular}{l ccccccc c}
\toprule
\multirow{3}{*}{\textbf{Method}} & \multicolumn{7}{c}{\textbf{OOD Datasets}} & \multirow{3}{*}{\textbf{Average}} \\
\cmidrule(lr){2-8}
& iNaturalist & ImageNet-O & NINCO & OpenImage-O & Places & SUN & Textures & \\
& FPR$\downarrow$ / AUC$\uparrow$ & FPR$\downarrow$ / AUC$\uparrow$ & FPR$\downarrow$ / AUC$\uparrow$ & FPR$\downarrow$ / AUC$\uparrow$ & FPR$\downarrow$ / AUC$\uparrow$ & FPR$\downarrow$ / AUC$\uparrow$ & FPR$\downarrow$ / AUC$\uparrow$ & FPR$\downarrow$ / AUC$\uparrow$ \\
\midrule
MSP & 53.9 / 88.3 & 100.0 / 29.4 & 60.9 / 83.8 & 68.2 / 83.8 & 72.9 / 80.4 & 70.1 / 81.5 & 67.1 / 80.4 & 70.4 / 75.4 \\
Energy & 54.4 / 90.5 & 100.0 / 41.7 & 49.8 / 87.5 & 65.3 / 86.9 & 66.6 / 83.8 & 59.4 / 86.4 & 53.1 / 86.6 & 64.1 / 80.5 \\
MaxLogit & 51.7 / 91.0 & 100.0 / 40.6 & 51.5 / 87.5 & 64.7 / 87.2 & 66.7 / 83.8 & 61.2 / 86.3 & 55.6 / 86.2 & 64.5 / 80.4 \\
Mahalanobis & 94.2 / 63.3 & 67.1 / 80.6 & 70.9 / 71.7 & 85.7 / 71.3 & 97.5 / 51.0 & 97.5 / 51.4 & 44.8 / 89.7 & 79.7 / 68.4 \\
ViM & 72.6 / 87.1 & 85.0 / 70.3 & 48.9 / 87.8 & 59.8 / 88.5 & 83.8 / 77.9 & 82.5 / 80.6 & 15.2 / 96.8 & 64.0 / 84.1 \\
ReAct & 30.9 / 94.6 & 99.9 / 46.5 & 41.8 / 89.3 & 55.4 / 89.5 & 54.5 / 87.9 & 46.0 / 90.3 & 46.6 / 90.0 & 53.6 / 84.0 \\
FDBD & 41.0 / 93.5 & 100.0 / 59.3 & 41.6 / 89.9 & 57.0 / 90.0 & 67.8 / 83.7 & 61.9 / 86.5 & 38.1 / 91.9 & 58.2 / 85.0 \\
Neco & 83.5 / 69.4 & 73.2 / 75.8 & 59.4 / 74.0 & 72.0 / 77.3 & 88.6 / 66.3 & 87.4 / 66.6 & 17.5 / 95.7 & 68.8 / 75.0 \\
WDiscOOD & 92.0 / 66.0 & 64.1 / 82.2 & 68.5 / 74.7 & 83.9 / 72.0 & 96.8 / 53.7 & 96.6 / 54.6 & 38.8 / 91.6 & 77.2 / 70.7 \\

KNN & 62.6 / 84.9 & 64.9 / 84.0 & 42.5 / 86.1 & 65.9 / 82.3 & 80.0 / 73.6 & 72.1 / 79.4 & 11.9 / 97.2 & 57.1 / 83.9 \\
\midrule
\rowcolor{gray!10} \textbf{D-KNN (Ours)} & \textbf{47.0 / 90.8} & \textbf{74.5 / 79.1} & \textbf{40.5 / 89.1} & 
\textbf{63.1 / 83.9} & \textbf{74.8 / 76.7} & \textbf{66.8 / 81.3} & \textbf{15.9 / 95.6} & \textbf{54.6 / 85.2} \\
\bottomrule
\end{tabular}
}
\caption{\textbf{Detailed OOD Detection Results on ImageNet-1K (ResNet-50).} \textbf{Metric:} FPR95 ($\downarrow$) / AUC ($\uparrow$). All results are percentages (\%).}
\label{tab:imagenet_resnet50_detailed}
\end{table*}

\begin{table*}[t]
\centering
\resizebox{\textwidth}{!}{%
\begin{tabular}{ll ccccccc c}
\toprule
\multirow{3}{*}{\textbf{Model}} & \multirow{3}{*}{\textbf{Method}} & \multicolumn{7}{c}{\textbf{OOD Datasets}} & \multirow{3}{*}{\textbf{Average}} \\
\cmidrule(lr){3-9}
& & EMNIST & MNIST & SVHN & Texture & LSUN & LSUN-R & iSUN & \\
& & FPR$\downarrow$ / AUC$\uparrow$ & FPR$\downarrow$ / AUC$\uparrow$ & FPR$\downarrow$ / AUC$\uparrow$ & FPR$\downarrow$ / AUC$\uparrow$ & FPR$\downarrow$ / AUC$\uparrow$ & FPR$\downarrow$ / AUC$\uparrow$ & FPR$\downarrow$ / AUC$\uparrow$ & FPR$\downarrow$ / AUC$\uparrow$ \\
\midrule
\multirow{10}{*}{\rotatebox{90}{DenseNet-101}} 
& MSP & 83.9 / 75.3 & 82.9 / 76.5 & 81.8 / 75.4 & 84.8 / 71.4 & 60.4 / 85.6 & 87.2 / 67.1 & 88.0 / 68.5 & 81.3 / 75.7 \\
& Energy & 74.8 / 84.3 & 62.9 / 88.6 & 87.5 / 81.9 & 84.3 / 71.0 & 14.8 / 97.4 & 75.4 / 77.8 & 78.7 / 76.8 & 68.3 / 82.6 \\
& Mahalanobis & 83.4 / 84.1 & 98.6 / 72.3 & 54.6 / 88.1 & 28.3 / 92.2 & 96.5 / 44.4 & 38.3 / 92.1 & 36.1 / 92.4 & 62.3 / 80.5 \\
& MaxLogit & 74.6 / 84.1 & 64.0 / 88.2 & 85.4 / 81.9 & 83.5 / 71.2 & 16.9 / 97.1 & 76.0 / 77.4 & 79.1 / 76.5 & 68.6 / 82.3 \\
& ViM & 50.5 / 89.8 & 41.5 / 92.7 & 97.0 / 56.3 & 86.9 / 63.6 & 31.3 / 92.4 & 90.1 / 61.5 & 91.5 / 60.8 & 69.8 / 73.9 \\
& ReAct & 64.5 / 87.8 & 60.0 / 89.4 & 83.0 / 83.2 & 79.7 / 77.7 & 15.8 / 97.1 & 71.5 / 81.7 & 75.1 / 81.1 & 64.2 / 85.4 \\
& DICE & 26.0 / 95.3 & 10.6 / 98.1 & 59.4 / 88.6 & 61.5 / 77.1 & 0.9 / 99.7 & 55.0 / 88.2 & 52.5 / 88.5 & 38.0 / 90.8 \\
& Line & 1.1 / 99.3 & 0.01 / 99.7 & 32.9 / 92.9 & 35.1 / 89.1 & 3.1 / 99.3 & 34.6 / 93.5 & 29.2 / 94.5 & 19.4 / 95.5 \\
& FDBD & 64.8 / 87.6 & 56.9 / 89.8 & 53.8 / 90.4 & 45.8 / 91.1 & 22.3 / 96.0 & 62.7 / 86.6 & 62.0 / 87.3 & 52.6 / 89.8 \\

& KNN & 2.1 / 98.6 & 1.6 / 98.6 & 17.5 / 96.4 & 24.3 / 93.7 & 31.4 / 92.8 & 47.3 / 90.4 & 39.7 / 91.9 & 23.4 / 94.6 \\
\rowcolor{gray!10} & \textbf{D-KNN (Ours)} & \textbf{5.0 / 98.3} & \textbf{5.5 / 98.1} & \textbf{25.1 / 95.1} & \textbf{23.1 / 94.2} & \textbf{21.8 / 95.8} & \textbf{21.2 / 96.0} & \textbf{21.3 / 96.0} & \textbf{17.6 / 96.2} \\
\midrule
\multirow{11}{*}{\rotatebox{90}{ResNet18}} 
& MSP & 38.4 / 93.5 & 35.2 / 94.7 & 81.3 / 79.2 & 85.4 / 74.1 & 74.0 / 81.6 & 74.5 / 83.4 & 77.7 / 81.7 & 66.6 / 84.0 \\
& Energy & 24.1 / 96.2 & 19.1 / 96.9 & 81.0 / 81.1 & 85.6 / 74.0 & 70.5 / 84.0 & 72.3 / 85.3 & 74.8 / 83.6 & 61.1 / 85.9 \\
& Mahalanobis & 99.7 / 47.7 & 100.0 / 40.9 & 90.4 / 71.4 & 72.1 / 75.6 & 95.7 / 55.7 & 89.8 / 68.8 & 85.6 / 70.3 & 90.5 / 61.5 \\
& MaxLogit & 30.2 / 95.5 & 25.3 / 96.3 & 80.7 / 81.0 & 85.5 / 74.3 & 71.8 / 83.7 & 72.5 / 85.2 & 75.4 / 83.5 & 63.1 / 85.6 \\

& ViM & 91.6 / 67.4 & 96.9 / 62.1 & 92.8 / 62.1 & 89.2 / 62.2 & 95.4 / 58.5 & 91.0 / 63.3 & 90.3 / 64.2 & 91.0 / 62.8 \\
& ReAct & 24.3 / 96.1 & 19.3 / 96.9 & 81.1 / 81.0 & 85.5 / 75.6 & 70.5 / 84.3 & 72.6 / 85.4 & 74.9 / 83.9 & 61.2 / 86.2 \\
& DICE & 32.6 / 95.7 & 23.8 / 96.4 & 74.3 / 83.4 & 86.5 / 72.2 & 66.9 / 85.2 & 82.2 / 82.8 & 82.1 / 81.7 & 64.1 / 85.3 \\
& Line & 51.1 / 92.3 & 48.6 / 93.2 & 73.9 / 82.0 & 81.2 / 73.8 & 66.6 / 85.6 & 86.0 / 77.2 & 84.8 / 78.0 & 70.3 / 83.2 \\
& FDBD & 41.2 / 93.1 & 38.6 / 93.9 & 77.2 / 82.3 & 79.3 / 79.8 & 74.0 / 81.8 & 69.6 / 86.6 & 72.6 / 85.3 & 64.6 / 86.1 \\

& KNN & 66.8 / 89.7 & 83.5 / 87.9 & 54.8 / 86.7 & 63.0 / 83.0 & 71.4 / 80.9 & 70.5 / 86.9 & 69.6 / 85.9 & 68.5 / 85.9 \\
\rowcolor{gray!10} & \textbf{D-KNN (Ours)} & \textbf{45.5 / 92.3} & \textbf{54.9 / 91.2} & \textbf{61.7 / 86.2} & \textbf{63.0 / 83.7} & \textbf{67.8 / 82.4} & \textbf{71.2 / 85.8} & \textbf{71.9 / 84.6} & \textbf{62.3 / 86.6} \\
\bottomrule
\end{tabular}
}
\caption{\textbf{Detailed OOD Detection Results on CIFAR-100 (DenseNet-101 and ResNet18).} \textbf{Metric:} FPR95 ($\downarrow$) / AUC ($\uparrow$). All results are percentages (\%).}
\label{tab:combined_cifar100_results}
\end{table*}

\begin{table*}[t]
\centering
\resizebox{\textwidth}{!}{%
\begin{tabular}{l l ccccccc c}
\toprule
\multirow{3}{*}{\textbf{Setup}} & \multirow{3}{*}{\textbf{Method}} & \multicolumn{7}{c}{\textbf{OOD Datasets}} & \multirow{3}{*}{\textbf{Average}} \\
\cmidrule(lr){3-9}
& & EMNIST & MNIST & SVHN & Textures & LSUN-R & LSUN & iSUN & \\
& & FPR$\downarrow$ / AUC$\uparrow$ & FPR$\downarrow$ / AUC$\uparrow$ & FPR$\downarrow$ / AUC$\uparrow$ & FPR$\downarrow$ / AUC$\uparrow$ & FPR$\downarrow$ / AUC$\uparrow$ & FPR$\downarrow$ / AUC$\uparrow$ & FPR$\downarrow$ / AUC$\uparrow$ & FPR$\downarrow$ / AUC$\uparrow$ \\
\midrule
\multicolumn{10}{l}{\textit{\textbf{CIFAR-10}}} \\
\multirow{2}{*}{CE} & KNN & 30.9 / 95.3 & 24.2 / 96.3 & 31.4 / 94.9 & 40.6 / 92.3 & 27.5 / 95.0 & 26.8 / 95.3 & 31.0 / 93.9 & 30.3 / 94.7 \\
\rowcolor{gray!10} & \textbf{D-KNN} & \textbf{2.5 / 99.4} & \textbf{2.1 / 99.5} & \textbf{17.2 / 96.7} & \textbf{28.8 / 94.5} & \textbf{26.2 / 94.9} & \textbf{15.7 / 97.2} & \textbf{30.0 / 93.9} & \textbf{17.5 / 98.0} \\
\addlinespace[0.5em]
\multirow{2}{*}{SupCon} & KNN & 17.4 / 96.6 & 6.5 / 98.9 & 2.4 / 99.5 & 8.4 / 98.6 & 19.5 / 96.8 & 1.8 / 99.5 & 20.0 / 96.7 & 10.9 / 97.7 \\
\rowcolor{gray!10} & \textbf{D-KNN} & \textbf{12.9 / 97.5} & \textbf{3.5 / 99.3} & \textbf{1.2 / 99.7} & \textbf{6.4 / 98.8} & \textbf{20.0 / 96.6} & \textbf{1.3 / 99.6} & \textbf{21.4 / 96.5} & \textbf{9.5 / 98.3} \\

\midrule
\multirow{3}{*}{\textbf{Setup}} & \multirow{3}{*}{\textbf{Method}} & \multicolumn{7}{c}{\textbf{OOD Datasets}} & \multirow{3}{*}{\textbf{Average}} \\
\cmidrule(lr){3-9}
& & iNaturalist & ImageNet-O & NINCO & OpenImage-O & Places & SUN & Textures & \\
& & FPR$\downarrow$ / AUC$\uparrow$ & FPR$\downarrow$ / AUC$\uparrow$ & FPR$\downarrow$ / AUC$\uparrow$ & FPR$\downarrow$ / AUC$\uparrow$ & FPR$\downarrow$ / AUC$\uparrow$ & FPR$\downarrow$ / AUC$\uparrow$ & FPR$\downarrow$ / AUC$\uparrow$ & FPR$\downarrow$ / AUC$\uparrow$ \\
\midrule
\multicolumn{10}{l}{\textit{\textbf{ImageNet-1K}}} \\
\multirow{2}{*}{CE} & KNN & 62.6 / 84.9 & 64.9 / 84.0 & 42.5 / 86.1 & 65.9 / 82.3 & 80.0 / 73.6 & 72.1 / 79.4 & 11.9 / 97.2 & 57.1 / 83.9 \\
\rowcolor{gray!10} & \textbf{D-KNN} & \textbf{47.0 / 90.8} & \textbf{74.5 / 79.1} & \textbf{40.5 / 89.1} & \textbf{63.1 / 83.9} & \textbf{74.8 / 76.7} & \textbf{66.8 / 81.3} & \textbf{15.9 / 95.6} & \textbf{54.6 / 85.2} \\
\addlinespace[0.5em]
\multirow{2}{*}{SupCon} & KNN & 31.8 / 94.6 & 67.4 / 83.3 & 34.5 / 91.2 & 38.7 / 92.6 & 60.4 / 84.6 & 50.0 / 88.5 & 15.8 / 95.4 & 42.7 / 90.0 \\
\rowcolor{gray!10} & \textbf{D-KNN} & \textbf{22.6 / 95.2} & \textbf{72.2 / 81.5} & \textbf{31.7 / 92.3} & \textbf{34.1 / 93.0} & \textbf{55.6 / 85.6} & \textbf{46.4 / 88.9} & \textbf{19.0 / 94.4} & \textbf{40.2 / 90.1} \\
\bottomrule
\end{tabular}
}
\caption{\textbf{OOD Detection Results on CIFAR-10 and ImageNet-1K.} We compare the baseline KNN with our D-KNN under Cross-Entropy (CE) and Supervised Contrastive (SupCon) learning frameworks.}
\label{tab:main_results_refined}
\end{table*}

\begin{figure*}
    
    \centering
     \includegraphics[width=1\linewidth]{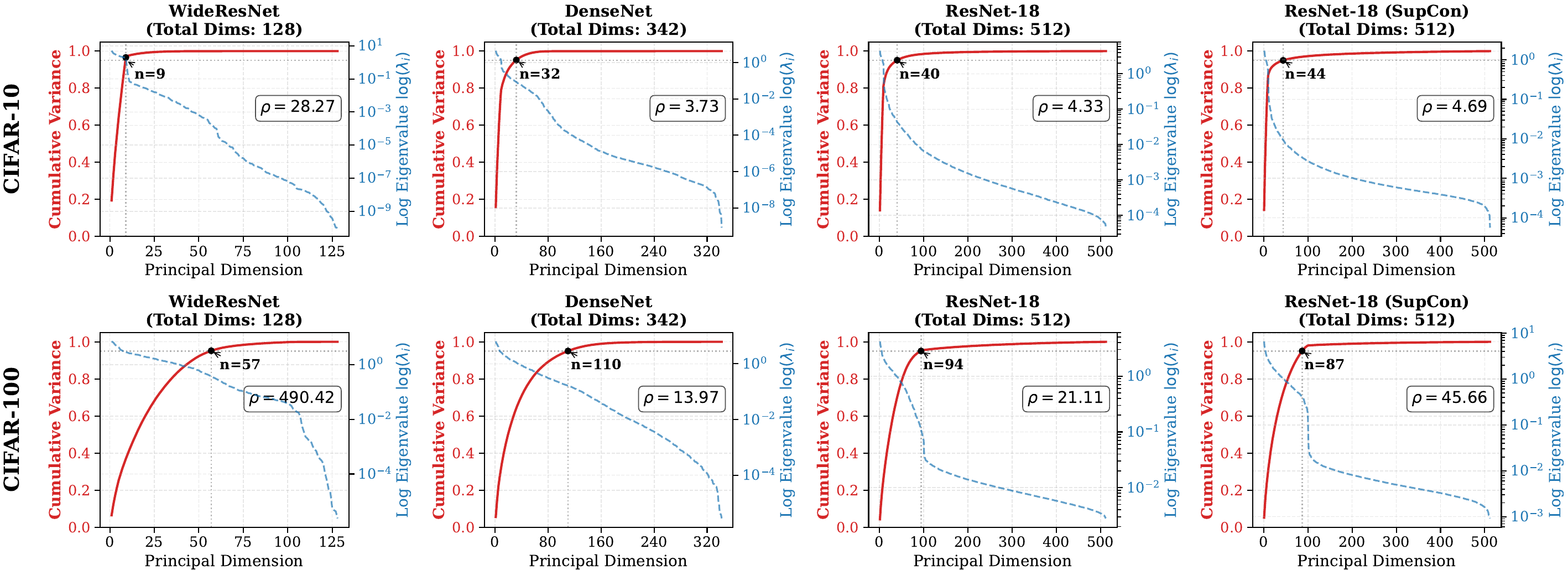}
    \caption{\textbf{Eigenspectrum Analysis and the Semantic Hegemony Ratio $\rho$.} This figure illustrates the exponential decay of eigenvalues across dimensions, revealing the severity of feature energy collapse into the principal subspace under various architectures.}
    \label{fig:shr}

\end{figure*}

\begin{figure}
    
    \centering
     \includegraphics[width=1\linewidth]{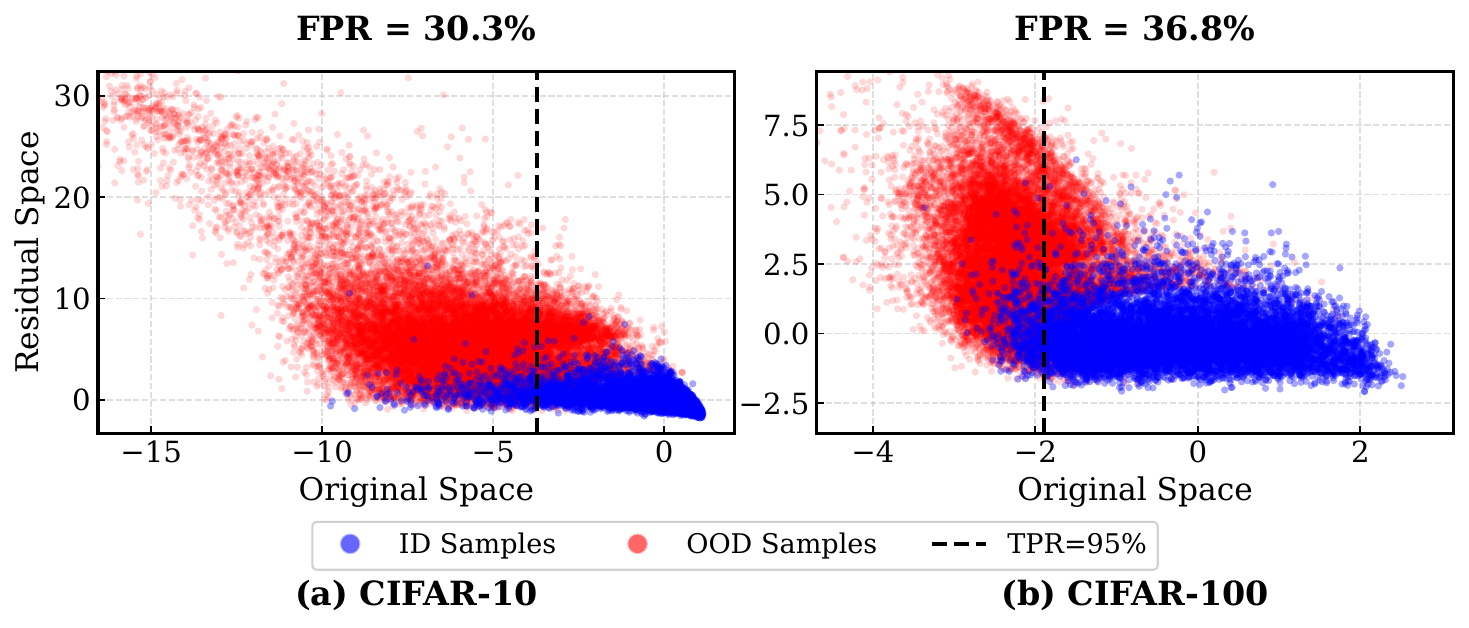}
    \caption{\textbf{Visualization Comparison of OOD Sample Distributions in Original and Residual Spaces.} This scatter plot clearly demonstrates how OOD samples incur semantic overlap within the original space and how geometric decoupling and separation are realized through the residual space.}
    \label{fig:residual}

\end{figure}

\begin{table*}[t]
\centering
\resizebox{\textwidth}{!}{%
\begin{tabular}{l ccccccc c}
\toprule
\multirow{3}{*}{\textbf{Method}} & \multicolumn{7}{c}{\textbf{ Corruptions Datasets}} & \multirow{3}{*}{\textbf{Average}} \\
\cmidrule(lr){2-8}
& Contrast & Frost & Gaussian & Glass Blur & Shot & Speckle & Zoom Blur & \\
& FPR$\downarrow$ / AUC$\uparrow$ & FPR$\downarrow$ / AUC$\uparrow$ & FPR$\downarrow$ / AUC$\uparrow$ & FPR$\downarrow$ / AUC$\uparrow$ & FPR$\downarrow$ / AUC$\uparrow$ & FPR$\downarrow$ / AUC$\uparrow$ & FPR$\downarrow$ / AUC$\uparrow$ & FPR$\downarrow$ / AUC$\uparrow$ \\
\midrule
MSP         & 87.3 / 70.7 & 86.1 / 69.5 & 81.8 / 79.7 & 85.0 / 72.6 & 82.7 / 78.6 & 81.7 / 78.3 & 85.7 / 68.7 & 84.3 / 74.0 \\
Energy      & 91.4 / 70.8 & 85.1 / 72.2 & 81.9 / 77.9 & 86.2 / 71.7 & 79.9 / 78.5 & 78.6 / 79.2 & 85.8 / 71.0 & 84.1 / 74.5 \\
MaxLogit    & 90.1 / 71.2 & 84.9 / 72.3 & 81.7 / 78.6 & 85.5 / 72.1 & 80.4 / 78.9 & 79.0 / 79.6 & 85.5 / 71.1 & 83.9 / 74.8 \\
Mahalanobis & 59.6 / 84.1 & 70.4 / 80.5 & 43.9 / 91.1 & 66.9 / 82.6 & 52.5 / 87.7 & 52.4 / 86.9 & 75.2 / 76.1 & 60.1 / 84.1 \\
ViM         & 74.6 / 69.2 & 87.4 / 63.5 & 53.8 / 84.2 & 89.4 / 60.7 & 53.9 / 83.2 & 56.8 / 81.8 & 92.0 / 59.2 & 72.6 / 71.7 \\
ReAct       & 91.4 / 71.9 & 84.9 / 73.9 & 82.4 / 78.9 & 86.3 / 72.1 & 79.9 / 79.4 & 78.7 / 80.1 & 85.8 / 71.9 & 84.2 / 75.5 \\
Neco        & 90.0 / 71.5 & 84.3 / 72.7 & 79.5 / 79.3 & 85.0 / 72.6 & 78.5 / 79.7 & 76.7 / 80.5 & 85.3 / 71.5 & 82.7 / 75.4 \\
KNN         & 47.0 / 84.9 & 58.2 / 82.6 & 31.5 / 93.4 & 60.0 / 83.5 & 36.4 / 91.8 & 32.6 / 92.2 & 68.5 / 76.6 & 47.7 / 86.4 \\
\midrule
\rowcolor{gray!10} \textbf{D-KNN (Ours)} & \textbf{43.0 / 86.7} & \textbf{58.1 / 83.5} & \textbf{24.0 / 94.9} & \textbf{58.0 / 84.9} & \textbf{30.0 / 93.3} & \textbf{29.4 / 93.2} & \textbf{68.6 / 77.6} & \textbf{44.5 / 87.7} \\
\bottomrule
\end{tabular}
}
\caption{\textbf{OOD Detection Results on CIFAR-100-C.} \textbf{Metric:} FPR95 ($\downarrow$) / AUC ($\uparrow$). All results are percentages (\%) and rounded to one decimal place.}
\label{tab:cifar100c_wrn_7datasets_final}
\end{table*}

\section{Experiments}
\label{sec:experiments}

\subsection{Experimental Setup}

\noindent\textbf{Benchmarks, Models, and Datasets.} 
We conduct comprehensive evaluations on two standard scales. For the CIFAR benchmarks (CIFAR-10 and CIFAR-100~\cite{cifar}), we employ WideResNet-28-10, ResNet-18, and DenseNet-101 as backbones. The standard OOD test suite comprises SVHN~\cite{SVHN}, LSUN (Crop \& Resize)~\cite{LSUN}, iSUN~\cite{iSUN}, and Textures~\cite{Textures}. 
Beyond standard protocols, we introduce two specific evaluation criteria: 
(1) \textbf{Simplicity Paradox:} We explicitly include structurally simple datasets, MNIST~\cite{lecun2002gradient} and EMNIST~\cite{cohen2017emnist}, to test robustness against spectral bias.
(2) \textbf{Sensor Failure Robustness:} To simulate hardware anomalies in real-world deployments, we utilize CIFAR-100-C~\cite{cifar100-c}, specifically selecting high-frequency noise corruptions as OOD datasets.
For the large-scale ImageNet benchmark, we utilize a pre-trained ResNet-50 on ImageNet-1K~\cite{deng2009imagenet}. The OOD evaluation includes massive-scale datasets (iNaturalist~\cite{van2018inaturalist}, Places~\cite{zhou2017places}, SUN~\cite{xiao2010sun}, Textures~\cite{Textures}) and distinctively challenging recent benchmarks (ImageNet-O~\cite{hendrycks2021natural}, NINCO~\cite{bitterwolf2023or}, OpenImage-O~\cite{yang2022openood}) to assess robustness against hard adversarial and open-world shifts. All models are trained using standard Cross-Entropy (CE) loss unless specified otherwise.

\noindent\textbf{Baselines and Metrics.} 
We compare D-KNN against a diverse range of sota OOD detection frameworks: 
(1) \textit{Output-based methods:} MSP, Energy, MaxLogit~\cite{hendrycks2019scaling}, and ReAct~\cite{sun2021react}; 
(2) \textit{Feature-based methods:} Mahalanobis, ViM, KNN, and DICE~\cite{sun2022dice}; 
(3) \textit{Recent advances:} Line~\cite{ahn2023line}, FDBD~\cite{liu2024fast}, NeCo~\cite{yang2023neural}, and WDiscOOD~\cite{chen2023wdiscood}. 
Performance is quantified using the False Positive Rate at 95\% True Positive Rate (\textbf{FPR95}) and the Area Under the Receiver Operating Characteristic curve (\textbf{AUROC}).

\noindent\textbf{Implementation Details.} 
For D-KNN, we extract features from the penultimate layer of pre-trained networks. The dimensionality of the principal subspace, $d$, is determined based on dataset complexity: for CIFAR-10, we align with the theoretical limit of Neural Collapse and set $d = C - 1$ ($C$ is the class count); for CIFAR-100 and ImageNet, where the feature spectrum exhibits a heavier tail, we determine $d$ adaptively to preserve 95\% of the cumulative spectral variance. The nearest neighbor count $k$ is consistent with the standard KNN baseline and we uniformly set $\alpha=0.5$ across all benchmarks. 

\subsection{Experimental Results and Analysis}
\label{sec:main_results}

\noindent\textbf{Main Results on CIFAR Benchmarks.} 
We first evaluate D-KNN on the standard CIFAR benchmarks using WideResNet. Table \ref{tab:combined_results_updated} details the performance against competitive baselines, where D-KNN achieves consistent improvements: on CIFAR-10, it reduces the average FPR95 to \textbf{17.5\%}, significantly outperforming the best baseline (Energy: 23.3\%) and standard KNN (30.3\%), while on CIFAR-100, it achieves an average FPR95 of \textbf{26.7\%}, surpassing KNN (36.8\%) by a large margin of \textbf{10.1\%}. A critical observation from Table \ref{tab:combined_results_updated} is the failure of standard feature-based methods on structurally simple datasets, defined here as the \textit{Simplicity Paradox}; for instance, on CIFAR-100 vs. MNIST, standard KNN yields a high FPR95 of 31.3\%, whereas D-KNN achieves near-perfect separation with an FPR95 of \textbf{2.3\%} (\textbf{1.9\%} on EMNIST). This substantial gain confirms that incorporating the residual subspace effectively corrects the biases inherent in pure semantic distances, capturing distinct structural signatures that remain hidden in the principal subspace due to \textit{Semantic Hegemony}.

\noindent\textbf{Scalability to Large-Scale ImageNet.} 
To verify the scalability of our approach, we evaluate performance on the ImageNet-1K benchmark using ResNet-50, with results summarized in Table \ref{tab:imagenet_resnet50_detailed}. In this challenging, high-dimensional scenario where density-based methods like Mahalanobis struggle (Avg. FPR95 79.7\%) due to precision matrix estimation difficulties, D-KNN achieves the best average performance with an FPR95 of \textbf{54.6\%}, outperforming the strong KNN baseline (57.1\%) and Energy (64.1\%). Notably, on fine-grained OOD datasets like iNaturalist, D-KNN reduces the error rate significantly from 62.6\% to \textbf{47.0\%}, and achieves an impressive \textbf{15.9\%} on Textures. This evidence suggests that the residual subspace contains critical fine-grained structural information that helps distinguish visually similar classes even amidst high-dimensional open-world shifts, effectively elevating hidden anomalies out of the geometric blind spot.

\noindent\textbf{Robustness across Architectures and Mechanism Analysis.} 
We further investigate the benefits of Dual-Space Calibration across different architectures by evaluating DenseNet and ResNet-18 on CIFAR-100, as shown in Table \ref{tab:combined_cifar100_results}. While D-KNN consistently improves upon the KNN baseline regardless of the backbone, we observe that the gain is positively correlated with model capacity: for DenseNet, D-KNN reduces the average FPR95 from 23.4\% to \textbf{17.6\%}, whereas ResNet-18 exhibits much more severe baseline failure (Avg. FPR95 68.5\%) and a more modest recovery to \textbf{62.3\%}. This discrepancy reveals a deep geometric mechanism: high-capacity models produce a \textit{Neural Collapse} state that more closely approaches the theoretical limit, creating an exceptionally clean residual subspace with a high signal-to-noise ratio, while low-capacity models like ResNet-18 yield cluttered distributions where the boundary between semantics and structure remains blurred. Thus, while Semantic Hegemony is an intrinsic byproduct of deep feature representations trained with Cross-Entropy, D-KNN proactively exploits this extreme energy compression to achieve high-precision geometric decoupling.

\subsection{Analysis and Discussion}

\noindent\textbf{Geometric Mechanism: Unveiling the Blind Spot}
To deeply understand the sources of performance gain we investigate the geometric properties of the feature space through spectral analysis and visualization. As shown in Figure \ref{fig:shr} the eigenspectrum of deep features exhibits a rapid decay trend. Under the WideResNet architecture the Semantic Hegemony Ratio $\rho$ reaches 28.27 for CIFAR-10 while it escalates to 490.42 for CIFAR-100. This indicates that driven by the Neural Collapse phenomenon the ratio approaches infinity as the intra-class variance of features tends toward zero causing the vast majority of feature energy to be compressed into an extremely low-dimensional principal subspace. In this scenario the standard metric is dominated by principal components. Considering that ood samples may possess subtle yet crucial structural anomalies in the residual subspace the standard nearest neighbor algorithm will misidentify them as id data because the magnitude of such anomalies is far smaller than the principal space vectors which leads to a geometric blind spot. The visualization results in Figure \ref{fig:residual} confirm that ood data and id data overlap significantly in the original space whereas our method successfully establishes a clear separation boundary by explicitly isolating and amplifying the residual signals.

\begin{figure}
    
    \centering
     \includegraphics[width=1\linewidth]{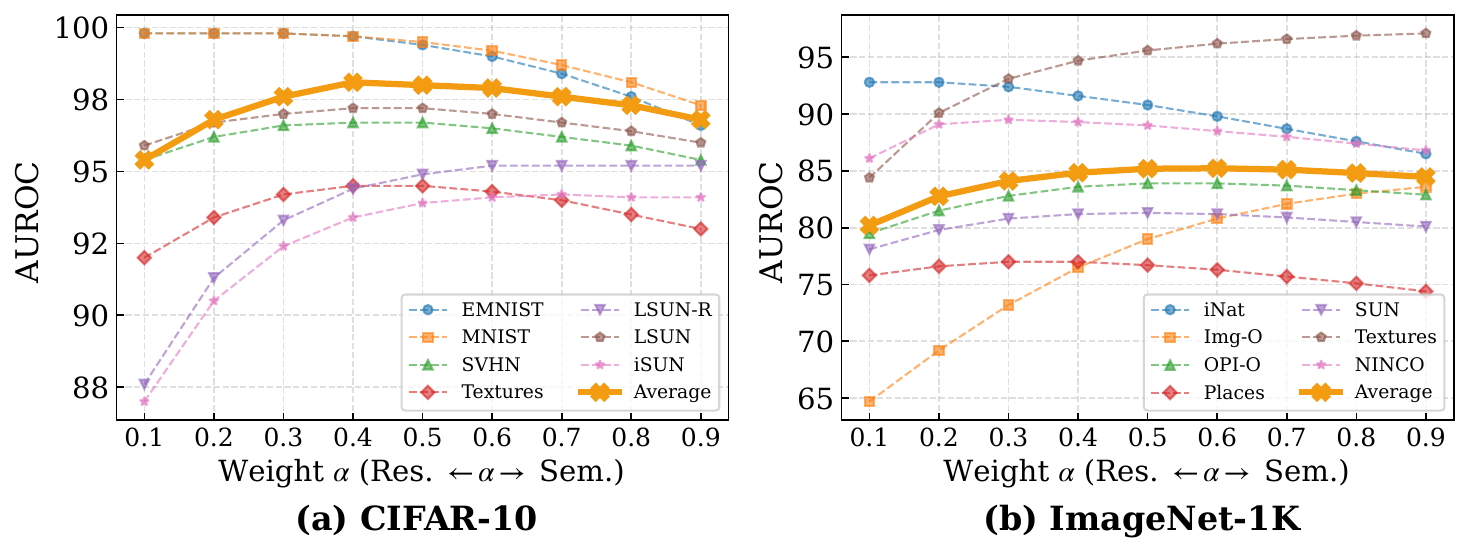}
    \caption{\textbf{Sensitivity Analysis of Detection Performance to the Fusion Weight Parameter $\alpha$ across Different Datasets.} This line chart demonstrates that D-KNN exhibits a stable inverted U-shaped performance distribution across various tasks, reflecting the method's low dependence on rigorous parameter tuning.}
    \label{fig:spectrum}

\end{figure}

\noindent\textbf{ Robustness and Generalization}
We further evaluate the stability under different parameters and training paradigms. Figure \ref{fig:spectrum} analyzes the impact of the weighting parameter on performance which balances the contributions of semantics and residuals. Across multiple datasets the curves exhibit a stable inverted U-shaped distribution. Specifically for CIFAR-10 the performance is most robust in the interval from 0.4 to 0.6. For ImageNet-1K the optimal interval shifts slightly toward higher residual weights primarily concentrated between 0.5 and 0.7 due to the increased complexity of the semantic space. This stability proves that there is no need for tedious fine-tuning for specific datasets. Furthermore we extend the framework to Supervised Contrastive Learning. As shown in Table \ref{tab:main_results_refined} contrastive learning models often exhibit more severe feature collapse than cross-entropy models. Although the baseline itself possesses strong detection capabilities it remains plagued by the simplicity paradox. Our method effectively alleviates this issue further reducing the false positive rate on digit datasets to 3.5\%. This demonstrates that the framework is a universal plugin independent of architecture and training paradigm.

\noindent\textbf{ Application: Sensor Failure Detection}
We simulate real-world hardware anomalies on the corruption dataset to evaluate reliability in safety-critical scenarios. As reported in Table \ref{tab:cifar100c_wrn_7datasets_final} traditional methods perform poorly when facing high-frequency corruptions such as Gaussian noise and shot noise. In contrast our method exhibits stronger robustness across all corruption types with an average false positive rate of 44.5\% and an area under the curve of 87.7\%. Particularly under Gaussian noise we significantly reduce the false positive rate from 31.5\% to 24.0\% and elevate the area under the curve to 94.9\%. Since sensor noise primarily manifests as high-frequency disturbances in the residual subspace the sensitivity of our method to structural residuals makes it an ideal health monitor for visual sensors.

\section{Conclusion}
In this study, we revisit the geometric properties of deep features and identify Semantic Hegemony, induced by Neural Collapse, as the root cause of the Geometric Blind Spot in existing OOD detection methods. This mechanism explains the counter-intuitive Simplicity Paradox, where standard methods fail to detect structurally simple or noisy anomalies.
To address this issue, we propose D-KNN that explicitly decouples the feature space into orthogonal principal and residual views. Through the Dual-Space Calibration mechanism, the numerically suppressed residual signals are successfully revitalized.
Extensive experiments demonstrate that D-KNN not only achieves SOTA performance on standard CIFAR and ImageNet benchmarks but also exhibits remarkable universality and robustness. On the CIFAR-100 benchmark, the method achieves an average FPR95 of 26.7\% and an average AUROC of 93.7\%. The method remains effective across diverse model architectures and training objectives, successfully resolving the failure modes of baseline methods on simple anomalies such as MNIST, where it reduces the FPR95 to 2.3\%. Furthermore, it demonstrates high reliability in sensor failure scenarios, achieving an AUROC of 94.9\% under Gaussian noise. Our findings suggest that in the era of deep representation learning, the often-neglected residual subspace holds critical geometric information essential for the safety and reliability of AI systems.
\bibliographystyle{named}
\bibliography{ijcai26}
\end{document}